%% file: ICML-Version-Shortened.tex
\documentclass{article}

\usepackage{microtype}
\usepackage{graphicx}
\usepackage{subfigure}
\usepackage{booktabs} 

\input{math_commands.tex}

\usepackage[colorlinks=true,linkcolor=red,citecolor = blue]{hyperref}%
\usepackage{algorithmic}
\usepackage{textcomp}
\usepackage{epstopdf}
\usepackage{times}
\usepackage{amssymb}
\usepackage{amsthm}

\usepackage{xcolor}

\makeatletter
\newtheorem*{rep@theorem}{\rep@title}
\newcommand{\newreptheorem}[2]{%
\newenvironment{rep#1}[1]{%
 \def\rep@title{#2 \ref{##1}}%
 \begin{rep@theorem}}%
 {\end{rep@theorem}}}
\makeatother

\newtheorem{theorem}{Theorem}
\newreptheorem{theorem}{Theorem}

\newtheorem{lemma}{Lemma}
\newreptheorem{lemma}{Lemma}

\newreptheorem{proposition}{Proposition}

\newtheorem{corollary}{Corollary}

\newtheorem{remark}{Remark}
\newcommand{\nn}{\nonumber}

\usepackage[accepted]{icml2021}

\icmltitlerunning{Generalization Error of Gibbs Algorithm}

\begin{document}

\twocolumn[
\icmltitle{Characterizing the Generalization Error of Gibbs Algorithm with Symmetrized KL information}




\icmlsetsymbol{equal}{*}

\begin{icmlauthorlist}
\icmlauthor{Gholamali Aminian}{equal,ucl}
\icmlauthor{Yuheng Bu}{equal,mit}
\icmlauthor{Laura Toni}{ucl}
\icmlauthor{Miguel R. D. Rodrigues}{ucl}
\icmlauthor{Gregory Wornell}{mit}
\end{icmlauthorlist}

\icmlaffiliation{mit}{Department of Electrical Engineering and Computer Science, Massachusetts Institute of Technology, Cambridge, USA}
\icmlaffiliation{ucl}{Department of Electronic and Electrical Engineering University College London, London, UK}

\icmlcorrespondingauthor{Gholamali Aminian}{g.aminian@ucl.ac.uk}
\icmlcorrespondingauthor{Yuheng Bu}{buyuheng@mit.edu}

\icmlkeywords{Machine Learning, ICML}

\vskip 0.3in
]



\printAffiliationsAndNotice{\icmlEqualContribution} 

\begin{abstract}
Bounding the generalization error of a supervised learning algorithm is one of the most important problems in learning theory, and various approaches have been developed. However, existing bounds are often loose and lack of guarantees. As a result, they may fail to characterize the exact generalization ability of a learning algorithm.
Our main contribution is an exact characterization of the expected generalization error of the well-known Gibbs algorithm in terms of symmetrized KL information between the input training samples and the output hypothesis. Such a result can be applied to tighten existing expected generalization error bound. Our analysis provides more insight on the fundamental role the symmetrized KL information plays in controlling the generalization error of the Gibbs algorithm.
\end{abstract}

\section{Introduction}\label{Sec:Introduction}

Evaluating the generalization error of a learning algorithm is one of the most important challenges in statistical learning theory. Various approaches have been developed~\cite{rodrigues2021information}, including VC dimension-based bounds~\cite{vapnik1999overview}, algorithmic stability-based bounds ~\cite{bousquet2002stability}, algorithmic robustness-based bounds ~\cite{xu2012robustness}, PAC-Bayesian bounds~\cite{mcallester2003pac}, and information-theoretic bounds~\cite{xu2017information}. 

However, upper bounds on generalization error may not entirely capture the generalization ability of a learning algorithm. One apparent reason is the tightness issue, some upper bounds \cite{anthony2009neural} can be far away from the true generalization error or even vacuous when evaluated in practice. More importantly, existing upper bounds do not fully characterize all the aspects that could influence the generalization error of a supervised learning problem. For example, VC dimension-based bounds depend only on the hypothesis class, and algorithmic stability-based bounds only exploit the properties of the learning algorithm. As a consequence, both methods fail to capture the fact that generalization behavior depends strongly on the interplay between the hypothesis class, learning algorithm, and the underlying data-generating distribution, as shown in \cite{zhang2016understanding}.
This paper overcomes the above limitations by deriving an exact characterization of the generalization error for a specific supervised learning algorithm, namely the Gibbs algorithm.




\subsection{Problem Formulation}

Let $S = \{Z_i\}_{i=1}^n$ be the training set, where each $Z_i$ is defined on the same alphabet $\mathcal{Z}$. Note that $Z_i$ is not required to be i.i.d generated from the same data-generating distribution $P_Z$, and we denote the joint distribution of all the training samples as $P_S$.
We denote the hypotheses by $w \in \mathcal{W}$, where $\mathcal{W}$ is a hypothesis class. The performance of the hypothesis is measured by a non-negative loss function $\ell:\mathcal{W} \times \mathcal{Z}  \to \mathbb{R}_0^+$, and we can define the empirical risk and the population risk associated with a given hypothesis $w$ as  $L_E(w,s)\triangleq\frac{1}{n}\sum_{i=1}^n \ell(w,z_i)$
and $L_P(w,P_S)\triangleq  \mathbb{E}_{P_S}[L_E(w,S)],$ respectively.

A learning algorithm can be modeled as a randomized mapping from the training set $S$ onto an hypothesis $W\in\mathcal{W}$  according to the conditional distribution $P_{W|S}$. Thus, the expected generalization error that quantifies the degree of over-fitting can be written as
\begin{equation}\label{Eq: expected GE}
\overline{\text{gen}}(P_{W|S},P_S)\triangleq\mathbb{E}_{P_{W,S}}[ L_P(W,P_S)-L_E(W,S)],
\end{equation}
where the expectation is taken over the joint distribution $P_{W,S} =  P_{W|S}\otimes P_S$. 

The $(\alpha,\pi(w),f(w,s))$-Gibbs distribution, which was first investigated by \cite{gibbs1902elementary}, is defined as:
\begin{equation}\label{Eq: Gibbs Solution}
    P_{{W}|S}^\alpha (w|s) \triangleq \frac{\pi({w}) e^{-\alpha f(w,s)}}{V(s,\alpha)},\quad \alpha\ge 0,
\end{equation}
where $\alpha$ is the inverse temperature, $\pi(w)$ is an arbitrary prior distribution of $W$, $f(w,s)$ is energy function, and $V(s,\alpha) \triangleq \int \pi(w) e^{-\alpha f(w,s)} dw$ is the partition function. 

If  $P$ and $Q$ are probability measures over the space $\mathcal{X}$, and $P$ is absolutely continuous with respect to $Q$, the Kullback-Leibler (KL) divergence between $P$ and $Q$ is given by
$D(P\|Q)\triangleq\int_\mathcal{X}\log\left(\frac{dP}{dQ}\right) dP$. If $Q$ is also absolutely continuous with respect to $P$, the symmetrized KL divergence (a.k.a. Jeffrey's divergence~ \cite{jeffreys1946invariant}) is 
\begin{equation}
D_{\mathrm{SKL}}(P\|Q)\triangleq D(P \| Q) + D(Q\|P).
\end{equation}

The mutual information between two random variables $X$ and $Y$ is defined as the KL divergence between the joint distribution and product-of-marginal
distribution $I(X;Y)\triangleq D(P_{X,Y}\|P_X\otimes P_{Y})$, or equivalently, the conditional KL divergence between $P_{Y|X}$ and $P_Y$ averaged over $P_X$, $D(P_{Y|X} \| P_Y|P_{X})\triangleq\int_\mathcal{X}D(P_{Y|X=x} \| P_Y) dP_{X}(x)$. By swapping the role of $P_{X,Y}$ and $P_X\otimes P_{Y}$ in mutual information, we get the lautum information introduced by \cite{palomar2008lautum}, $L(X;Y)\triangleq D(P_X\otimes P_{Y}\| P_{X,Y})$. Finally, the symmetrized KL information between $X$ and $Y$ is given by \cite{aminian2015capacity}:
\begin{equation}
   I_{\mathrm{SKL}}(X;Y)\triangleq  I(X;Y)+ L(X;Y). 
\end{equation}
Throughout the paper, upper-case letters denote random variables, lower-case letters denote the realizations of random variables, and calligraphic letters denote sets. 
All the logarithms are the natural ones, and all the information measure units are nats. $\mathcal{N}(\mu,\Sigma)$ denotes a Gaussian distribution with mean $\mu$ and covariance matrix $\Sigma$.

\subsection{Contributions}
The core contribution of this paper (see Theorem \ref{Theorem: Gibbs Result}) is an \emph{exact} characterization of the expected generalization error for the Gibbs algorithm in terms of symmetrized KL information between the input training samples $S$ and the output hypothesis $W$, as follows:
\begin{equation*}
    \overline{\text{gen}}(P_{{W}|S}^\alpha,P_S) =  \frac{I_{\mathrm{SKL}}(W;S)}{\alpha}.
\end{equation*}
This result highlights the fundamental role of such an information quantity in learning theory that does not appear to have been recognized before. We also discuss some general properties of the symmetrized KL information, which could be used to prove the non-negativity and concavity of the expected generalization error for the Gibbs algorithm.

Building upon this result, we further expand our contributions by tightening existing expected generalization error for Gibbs algorithm under i.i.d and sub-Gaussian assumptions by combining our symmetrized KL information characterization with existing bounding techniques. 


\subsection{Motivations for Gibbs Algorithm}
As we discuss below, the choice of the Gibbs algorithm is not arbitrary since it arises naturally in many different applications and is sufficiently general to model many learning algorithms used in practice:

\textbf{Empirical Risk Minimization:}  The $(\alpha,\pi(w),L_E(w,s))$-Gibbs algorithm can be viewed as a randomized version of the empirical risk minimization (ERM) algorithm if we specify the energy function $f(w,s) = L_E(w,s)$. As the inverse temperature $\alpha \to \infty$, the prior distribution $\pi(w)$ becomes negligible, and the Gibbs algorithm converges to the standard ERM algorithm.

\textbf{Information Risk Minimization:} The \small$(\alpha,\pi(w),L_E(w,s))$\normalsize-Gibbs algorithm is the solution to the regularized ERM problem by considering conditional KL-divergence $D(P_{W|S}\|\pi(W)|P_S)$, as a regularizer to penalize over-fitting in the information risk minimization framework~\cite{xu2017information,zhang2006information,zhang2006E}. 

\textbf{SGLD Algorithm:} 
The Stochastic Gradient Langevin Dynamics (SGLD) can be viewed as the discrete version of the continuous-time Langevin diffusion. In \cite{raginsky2017non}, it is proved that under some conditions on loss function, the learning algorithm induced by SGLD algorithm is close to $(\alpha,\pi(W_0),L_E(w_k,s))$-Gibbs distribution in 2-Wasserstein distance for sufficiently large iterations, where $\pi(W_0)$ is the distribution over hypothesis in the first step. Under some conditions on the loss function $\ell(w,z)$, \cite{chiang1987diffusion,markowich2000trend} shows that in the continuous-time Langevin diffusion, the stationary distribution of hypothesis $W$ is the Gibbs distribution.

\subsection{Other Related Works}\label{subsec:related_work}

\textbf{Information-theoretic generalization error bounds:} Recently, \cite{russo2019much,xu2017information} proposed to use the mutual information between the input training set and the output hypothesis to upper bound the expected generalization error. However, those bounds are known not to be tight, and multiple approaches have been proposed to tighten the mutual information-based bound. \cite{bu2020information} provides tighter bounds by considering the individual sample mutual information, \cite{asadi2018chaining,asadi2020chaining} propose using chaining mutual information,  and \cite{steinke2020reasoning,hafez2020conditioning,haghifam2020sharpened} advocate the conditioning and processing techniques. 
Information-theoretic generalization error bounds using other information quantities are also studied, such as, $f$-divergence~\cite{jiao2017dependence}, $\alpha$-R\'eyni divergence and maximal leakage~\cite{issa2019strengthened,esposito2019generalization}, Jensen-Shannon divergence~\cite{aminian2020jensen} and Wasserstein distance~\cite{lopez2018generalization,wang2019information,rodriguez2021tighter}. Using rate-distortion theory, \cite{masiha2021learning,bu2020information} provide information-theoretic generalization error upper bounds for model misspecification and model compression.  

\textbf{Generalization error of Gibbs algorithm:} Both information-theoretic and PAC-Bayesian approaches have been used to bound the generalization error of the Gibbs algorithm.  An information-theoretic upper bound with a convergence rate of $\mathcal{O}\left(1/n\right)$ is provided in \cite{raginsky2016information} for the Gibbs algorithm with bounded loss function, and PAC-Bayesian bounds using a variational approximation of Gibbs posteriors are studied in \cite{alquier2016properties}. \cite{kuzborskij2019distribution} focus on the excess risk of the Gibbs algorithm and a similar generalization bound with rate of $\mathcal{O}\left(1/n\right)$ is provided under sub-Gaussian assumption. Although these bounds are tight in terms of the sample complexity $n$, they become vacuous when the inverse temperature $\alpha \to \infty$, hence are unable to capture the behaviour of the ERM algorithm.

Our work differs from this body of research in the sense that we provide an exact characterization of the generalization error of the Gibbs algorithm in terms of the symmetrized KL information. Our work also further leverages this characterization to tighten existing expected generalization error bounds in literature.

\section{Generalization Error of Gibbs Algorithm}\label{Sec:Main Results}

Our main result, which characterizes the exact expected generalization error of the Gibbs algorithm with prior distribution $\pi(w)$, is as follows:
\begin{theorem}\label{Theorem: Gibbs Result}
For the $(\alpha,\pi(w),L_E(w,s))$-Gibbs algorithm,
the expected generalization error is given by
\begin{equation}
    \overline{\text{gen}}(P_{{W}|S}^\alpha,P_S) =  \frac{I_{\mathrm{SKL}}(W;S)}{\alpha}.
\end{equation}
\end{theorem}
\begin{proof}[\textbf{Sketch of Proof:}] It can be shown that the symmetrized KL information can be written as
\begin{equation}\label{eq:symetrized KL representation}
    I_{\mathrm{SKL}}(W;S)=\mathbb{E}_{P_{W,S}}[\log(P_{W|S}^\alpha)]-\mathbb{E}_{P_W \otimes P_S}[\log(P_{W|S}^\alpha)].
\end{equation}
Just like the generalization error, the above expression is the difference between 
the expectations of the same function evaluated under the joint distribution and the product-of-marginal distribution. Note that $P_{W,S}$ and $P_W \otimes P_S$ share the same marginal distribution, we have $\mathbb{E}_{P_{W,S}}[\log\pi(W)] = \mathbb{E}_{P_W}[\log\pi(W)]$, and 
$\mathbb{E}_{P_{W,S}}[\log V(S,\alpha)] = \mathbb{E}_{P_S}[\log V(S,\alpha)]$. Then,  combining \eqref{Eq: Gibbs Solution} with \eqref{eq:symetrized KL representation} completes the proof.  More details are provided in Appendix~\ref{app: Gibbs Algorithm}.
\end{proof}

To the best of our knowledge, this is the first exact characterization of the expected generalization error for the Gibbs algorithm. Note that Theorem~\ref{Theorem: Gibbs Result} only assumes that the loss function is non-negative, and it holds even for non-i.i.d training samples. 

\subsection{General Properties}\label{subsec:property}

By Theorem~\ref{Theorem: Gibbs Result}, some basic properties of the expected generalization error, including non-negativity and concavity can be proved directly from the properties of symmetrized KL information. 

The non-negativity of the expected generalization error, i.e., $\overline{\text{gen}}(P_{{W}|S}^\alpha,P_S)\ge 0$, follows by the non-negativity of the symmetrized KL information. Note that the non-negativity result obtained in \cite{kuzborskij2019distribution} requires more technical assumptions, including i.i.d samples and a sub-Gaussian loss function. 

It is shown in \cite{aminian2015capacity} that the symmetrized KL information $I_{\mathrm{SKL}}(X;Y)$ is a concave function of $P_X$ for fixed $P_{Y|X}$, and a convex function of $P_{Y|X}$ for fixed $P_X$. Thus, we have the following corollary.
\begin{corollary}\label{cor:concavity}
 For a fixed $(\alpha,\pi(w),L_E(w,s))$-Gibbs algorithm $P_{{W}|S}^\alpha$, the expected generalization error $\overline{\text{gen}}(P_{{W}|S}^\alpha,P_S)$ is a concave function of $P_S$.
\end{corollary}

The concavity of the generalization error for the Gibbs algorithm $P_{{W}|S}^\alpha$ can be immediately used to explain the well-known fact why training a model by mixing multiple datasets from different domains leads to poor generalization. Suppose that the data-generating distribution is domain-dependent, i.e., there exists a random variable $D$, such that $D \leftrightarrow S \leftrightarrow W$ holds. Then, $P_S = \mathbb{E}_{P_D}[P_{S|D}]$ can be viewed as the mixture of the data-generating distribution across all domains. From Corollary ~\ref{cor:concavity} and Jensen's inequality, we have
\begin{equation}
    \overline{\text{gen}}(P_{{W}|S}^\alpha,P_S) \ge \mathbb{E}_{P_D}\big[ \overline{\text{gen}}(P_{{W}|S}^\alpha, P_{S|D} )\big],
\end{equation}
which shows that the generalization error of Gibbs algorithm achieved with the mixture distribution $P_S$ is larger than the averaged generalization error for each $P_{S|D}$.

\subsection{Example: Mean Estimation}\label{sec:mean_example}
We now consider a simple learning problem, where the symmetrized KL information can be computed exactly, to demonstrate the usefulness of Theorem~\ref{Theorem: Gibbs Result}. All details are provided in Appendix~\ref{app:mean}.

Consider the problem of learning the mean $\vmu \in \mathbb{R}^d$ of a random vector $Z$ using $n$ i.i.d training samples $S=\{Z_i\}_{i=1}^n$. We assume that the covariance matrix of $Z$ satisfies  $\Sigma_Z = \sigma^2_Z I_d$ with unknown $\sigma^2_Z$. We adopt the mean-squared loss  $\ell(\vw,\vz) = \|\vz-\vw\|_2^2$, and assume a Gaussian prior for the mean $\pi (\vw) =  \mathcal{N}(\vmu_0,\sigma^2_0 I_d)$. If we set inverse-temperature $\alpha=\frac{n}{2\sigma^2}$, then the $(\frac{n}{2\sigma^2},\mathcal{N}(\vmu_0,\sigma^2_0 I_d),L_E(\vw,s))$-Gibbs algorithm is given by the following posterior distribution \cite{murphy2007conjugate},
\begin{equation}\label{equ:mean_alg}
    P_{W|S}^{\alpha}(\vw|Z^n)\sim \mathcal{N}\Big( \frac{\sigma_1^2 }{\sigma_0^2}\vmu_0 +\frac{\sigma_1^2 }{\sigma^2}\sum_{i=1}^n Z_i,\sigma_1^2 I_d\Big), 
\end{equation}
with $\sigma_1^2 = \frac{\sigma_0^2 \sigma^2}{n\sigma_0^2 +\sigma^2}.$

Since $P_{W|S}^{\alpha}$ is Gaussian, the mutual information and lautum information are given by 
\small
\begin{align}
    I(S;W) &= \frac{n d \sigma_0^2 \sigma_Z^2 }{(n\sigma_0^2 + \sigma^2)\sigma^2} - D\big(P_W\| \mathcal{N}(\vmu_W,\sigma_1^2 I_d) \big),  \\
    L(S;W) &= \frac{n d \sigma_0^2 \sigma_Z^2 }{(n\sigma_0^2 + \sigma^2)\sigma^2} + D\big(P_W\| \mathcal{N}(\vmu_W,\sigma_1^2 I_d)), 
\end{align}
\normalsize
with $\vmu_W = \frac{\sigma_1^2 }{\sigma_0^2}\vmu_0 + \frac{n\sigma_1^2  }{\sigma^2} \vmu.$  As we can see from the above expressions, symmetrized KL information $I_{\mathrm{SKL}}(W;S)$
is independent of the distribution of $P_Z$, as long as $\Sigma_Z = \sigma^2_Z I_d$.


From Theorem~\ref{Theorem: Gibbs Result}, the generalization error of this algorithm can be computed exactly as:
\begin{align}\label{eq:mean_truth}
  \overline{\text{gen}}(P_{W|S}^{\alpha},P_S) = \frac{I_{\mathrm{SKL}}(W;S)}{\alpha} 
  & = \frac{2 d \sigma_0^2 \sigma_Z^2 }{n(\sigma_0^2 +\frac{1}{2\alpha})},
\end{align}
which has the decay rate of $\mathcal{O}\left(1/n\right)$. As a comparison, the individual sample mutual information (ISMI) bound from \cite{bu2020tightening}, which is shown to be tighter than the mutual information-based bound in \citep[Theorem 1]{xu2017information}, gives a sub-optimal bound  with order $\mathcal{O}\left(1/\sqrt{n}\right)$, as $n\to \infty$, (see Appendix \ref{app:ISMI}).
\section{Tighter Expected Generalization Error Upper Bound}

In this section, we show that by combining Theorem~\ref{Theorem: Gibbs Result} with the information-theoretic bound proposed in \cite{xu2017information} under i.i.d and sub-Gaussian assumptions, we can provide a tighter generalization error upper bound for Gibbs algorithm. This bound quantifies how the generalization error of the Gibbs algorithm depends on the number of samples $n$, and is useful when  directly evaluating the symmetrized KL information is hard.
\begin{theorem}\label{Theorem: Sub Gaussian extension}(proved in Appendix~\ref{app: Average Upper Bound Details}) 
Suppose that the training samples $S=\{Z_i\}_{i=1}^n$ are i.i.d generated from the distribution $P_Z$, and the non-negative loss function $\ell(w,Z)$ is $\sigma$-sub-Gaussian on the left-tail \footnote{A random variable $X$ is $\sigma$-sub-Gaussian  if $\log \mathbb{E}[e^{\lambda(X-\mathbb{E}X)}] \le \frac{\sigma^2\lambda^2}{2}$, $\forall \lambda \in \mathbb{R} $, and $X$ is $\sigma$-sub-Gaussian on the left-tail if $\log \mathbb{E}[e^{\lambda(X-\mathbb{E}X)}] \le \frac{\sigma^2\lambda^2}{2}$, $\forall \lambda \leq 0 $.}
under distribution $P_Z$ for all $w\in \mathcal{W}$. If we further assume $C_E\le \frac{L(W;S)}{I(W;S)}$ for some $C_E \ge 0$, then for the $(\alpha,\pi(w),L_E(w,s))$-Gibbs algorithm, we have
\begin{equation}
     0\leq\overline{\text{gen}}(P_{{W}|S}^{\alpha},P_S) \leq \frac{2\sigma^2\alpha}{(1+C_E)n}.
\end{equation}
\end{theorem}
Theorem~\ref{Theorem: Sub Gaussian extension} establishes the  convergence rate $\mathcal{O}(1/n)$ of the generalization error of Gibbs algorithm with i.i.d training samples, and suggests that a smaller inverse temperature $\alpha$ leads to a tighter upper bound. Note that all the $\sigma$-sub-Gaussian loss functions are also $\sigma$-sub-Gaussian on the left-tail under the same distribution (loss function in Section~\ref{sec:mean_example} is $\sigma$-sub-Gaussian on the left-tail, but not sub-Gaussian). Therefore, our result also applies to any bounded loss function $\ell:\mathcal{W}\times \mathcal{Z}\rightarrow [a,b]$, since bounded functions are $(\frac{b-a}{2})$-sub-Gaussian.
\begin{remark}
[Previous Results] Using the fact that Gibbs algorithm is differentially private \cite{mcsherry2007mechanism} for bounded loss functions $\ell \in [0,1]$, directly applying \citep[Theorem~1]{xu2017information} gives a sub-optimal bound  $|\overline{\text{gen}}(P_{{W}|S}^\alpha,P_S) |\leq \sqrt{\frac{\alpha}{n}}$. By further exploring the bounded loss assumption using Hoeffding’s lemma, a tighter upper bound  $|\overline{\text{gen}}(P_{{W}|S}^\alpha,P_S) |\leq \frac{\alpha}{2n}$ is obtained in
 \cite{raginsky2016information}, which has the similar decay rate order of $\mathcal{O}\left(1/n\right)$.
In \citep[Theorem~1]{kuzborskij2019distribution}, the upper bound $\overline{\text{gen}}(P_{{W}|S}^\alpha,P_S) \leq \frac{4 \sigma^2 \alpha}{n}$ is derived with a different assumption, i.e., $\ell(W,z)$ is $\sigma$-sub-Gaussian under Gibbs algorithm $P_{{W}|S}^{\alpha}$. In Theorem~\ref{Theorem: Sub Gaussian extension}, we assume the loss function is $\sigma$-sub-Gaussian on left-tail under data generating distribution $P_Z$ for all $w\in \mathcal{W}$, which is more general as we discussed above. Our upper bound is also improved by a factor of $\frac{1}{2(1+C_E)}$ compared to the result in \cite{kuzborskij2019distribution}.
\end{remark}
\begin{remark}
[Choice of $C_E$] Since $L(W;S) > 0$ when $I(W;S)>0$, setting $C_E=0$ is always valid in Theorem~\ref{Theorem: Sub Gaussian extension}, which gives $\overline{\text{gen}}(P_{{W}|S}^\alpha,P_S) \leq \frac{2 \sigma^2 \alpha}{n}$. As shown in \citep[ Theorem~15]{palomar2008lautum}, $L(S;W)\ge I(S;W)$ holds for any Gaussian channel $P_{W|S}$. In addition, it is discussed in \citep[Example~1]{palomar2008lautum}, if either the entropy of training $S$ or the hypothesis $W$ is small, $I(S;W)$ would be smaller than $L(S;W)$ (as it is not upper-bounded by the entropy), which implies that the lautum information term is not negligible in general.
\end{remark}

\section{Conclusion}\label{Conc}
We provide an exact characterization of the expected generalization error for the Gibbs algorithm using symmetrized KL information. We demonstrate the versatility of our approach by tightening the existing information-theoretic expected generalization error upper bound. This work motivates further investigation of the Gibbs algorithm in a variety of settings, including extending our results to characterize the generalization ability of a over-parameterized Gibbs algorithm, which could potentially provide more understanding of the generalization ability for deep learning.

\section{Acknowledgment}

Yuheng Bu is supported, in part, by NSF under Grant CCF-1717610 and by the MIT-IBM Watson AI Lab. Gholamali Aminian is supported by the Royal Society Newton International Fellowship, grant no. NIF\textbackslash R1 \textbackslash 192656. 

\bibliography{Refs}
\bibliographystyle{icml2021}

\appendix

\section{Proof of Theorem~\ref{Theorem: Gibbs Result}}\label{app: Gibbs Algorithm}
We start with the following two Lemmas:
\begin{lemma}\label{lemma:loss}
We define the following $J_E(w,S)$ function as a proxy for the empirical risk, i.e., $J_E(w,S)\triangleq \frac{\alpha}{n}\sum_{i=1}^n \ell(w,Z_i)+g(w)+h(S)$, where $\alpha \in \mathbb{R}_0^+$, $g: \mathcal{W}\to \mathbb{R}$, $h: \mathcal{Z}^n\to \mathbb{R}$, and the function $J_P(w,\mu)\triangleq \mathbb{E}_{P_S}[J_E(w,S)]$ as a proxy for the population risk.
Then, 
\begin{equation}
    \mathbb{E}_{P_{W,S}}[J_P(W,\mu) - J_E(W,S)] = \alpha \cdot \overline{\text{gen}}(P_{W|S},P_S).
\end{equation}
\end{lemma}

\begin{proof}
\begin{align*}
&\mathbb{E}_{P_{W,S}}[J_P(W,\mu) - J_E(W,S)]\nn \\
&=\mathbb{E}_{P_{W,S}}\Big[ \mathbb{E}_{P_{Z^n}}[\frac{\alpha}{n}\sum_{i=1}^n \ell(W,Z_i)] -\frac{\alpha}{n}\sum_{i=1}^n \ell(W,Z_i)\Big] \nn\\
&\quad + \mathbb{E}_{P_{W}}\Big[g(W)+\mathbb{E}_{P_{S}}[h(S)]\Big]- \mathbb{E}_{P_{W,S}}\Big[g(W)+ h(S)\Big]\nn\\
& = \alpha \cdot \mathbb{E}_{P_{W,S}}[ L_P(W,\mu)-L_E(W,S)] \nn\\
& = \alpha \cdot \overline{\text{gen}}(P_{W|S},P_S).
\qedhere
\end{align*}
\end{proof}

\begin{lemma}\label{lemma:symmetric_KL}
Consider a learning algorithm $P_{W|S}$, if we set the function $J_E(w,z^n) = -\log P_{W|S}(w|s)$, 
then
\begin{equation}
    \mathbb{E}_{P_{W,S}}[J_P(W,\mu) - J_E(W,S)] = I_{\mathrm{SKL}}(W;S).
\end{equation}
\end{lemma}

\begin{proof}
\begin{align*}
    &I(W;S) + L(W;S) \nn \\
    &= \mathbb{E}_{P_{W,S}}\Big[\log \frac{P_{W|S}(W|S)}{P_W(W)}\Big] + \mathbb{E}_{P_W \otimes P_S}\Big[\log \frac{P_W(W)}{P_{W|S}(W|S)}\Big]\nn\\
    &= \mathbb{E}_{P_{W,S}}\Big[\log {P_{W|S}(W|S)}\Big] - \mathbb{E}_{P_W \otimes P_S}\Big[\log {P_{W|S}(W|S)}\Big]\nn \\
     &= \mathbb{E}_{P_{W,S}}[-\mathbb{E}_{P_{S}}[\log P_{W|S}(W|S)] +\log P_{W|S}(W|S)]\nn\\
    &=\mathbb{E}_{P_{W,S}}[J_P(W,\mu) - J_E(W,S)]. \qedhere
\end{align*}
\end{proof}

Considering Lemma~\ref{lemma:loss} and Lemma~\ref{lemma:symmetric_KL}, we just need to verify that the function $J_E(w,s) = -\log P_{{W}|S}(w|s)$ can be decomposed into $J_E(w,s) = \frac{\alpha}{n}\sum_{i=1}^n\ell(w,z_i) +g(w)+h(s)$, for $\alpha>0$. 
Note that
\begin{align*}
    -\log P_{{W}|S}^\alpha(w|s)
    =\alpha L_E(w,s) -\log \pi({w})  +\log V(s,\alpha),
\end{align*}
 then we have:
\begin{align}
    I_{\mathrm{SKL}}(W;S) &= \mathbb{E}_{P_{W,S}}[J_P(W,P_S) - J_E(W,S)] \nn \\
    &= \alpha \cdot \overline{\text{gen}}(P_{{W}|S}^\alpha,P_S).
\end{align}

\section{Example Details: Estimating the Mean of Gaussian }\label{app:mean}
\subsection{Generalization Error}
We first evaluate the generalization error of the learning algorithm in \eqref{equ:mean_alg} directly. Note that the output $W$  can be written as
\begin{equation}
    W = \frac{\sigma_1^2 }{\sigma_0^2}\vmu_0 +\frac{\sigma_1^2 }{\sigma^2}\sum_{i=1}^n Z_i +N,
\end{equation}
where $N\sim \mathcal{N}(0,\sigma^2_1 I_d)$ is independent from the training samples $S=\{Z_i\}_{i=1}^n$. Thus,
\begin{align}
   &\overline{\text{gen}}(P_{W|S},P_Z)\nn \\
   &=\mathbb{E}_{P_{W,S}}[ L_P(W,\mu)-L_E(W,S)] \nn \\
   & = \mathbb{E}_{P_{W,S}}\Big[ \mathbb{E}_{P_{\widetilde{Z}}}\big[\|W-\widetilde{Z}\|_2^2\big] - \frac{1}{n}\sum_{i=1}^n\|W-Z_i\|_2^2 \Big]\nn\\
   & \overset{(a)}{=} \mathbb{E}_{P_{W,Z_i}\otimes P_{\widetilde{Z}}}\Big[ (2W-\widetilde{Z}-Z_i)^\top(Z_i - \widetilde{Z}) \Big] \nn\\
   & = \mathbb{E}\Big[ 2(\frac{\sigma_1^2 }{\sigma_0^2}\vmu_0 +\frac{\sigma_1^2 }{\sigma^2}\sum_{i=1}^n Z_i +N)^\top(Z_i - \widetilde{Z})\nn \\
   &\qquad\qquad\qquad-(Z_i + \widetilde{Z})^\top(Z_i - \widetilde{Z}) \Big]\nn\\
   &\overset{(b)}{=} \frac{2\sigma_1^2}{\sigma^2}\mathbb{E}\Big[  Z_i ^\top(Z_i - \widetilde{Z}) \Big]\nn\\
   & = \frac{2d\sigma_1^2\sigma_Z^2}{\sigma^2}=\frac{2 d \sigma_0^2 \sigma_Z^2}{n\sigma_0^2 + \sigma^2},
\end{align}
where $\widetilde{Z}\sim \mathcal{N}(\vmu, \sigma_Z^2 I_d)$ denotes an independent copy of the training sample, $(a)$ follows due to the fact that $Z^n$ are i.i.d, and $(b)$ follows from the fact that $Z_i - \widetilde{Z}$ has zero mean, and it is only dependent on $Z_i$. 
\subsection{Symmetrized KL divergence}
The following lemma from \cite{palomar2008lautum} characterizes the mutual and lautum information for the Gaussian Channel.

\begin{lemma}{\citep[Theorem 14]{palomar2008lautum}}\label{lemma:Gaussian}
Consider the following model 
\begin{equation}
    \mY = \mA \mX+\mN_{\mathrm{G}},
\end{equation}
where $\mX \in \mathbb{R}^{d_X}$ denotes the input
random vector with zero mean (not necessarily
Gaussian), $\mA \in \mathbb{R}^{d_Y \times d_X}$ denotes the linear transformation undergone by the input, $\mY\in \mathbb{R}^{d_Y}$ is the 
output vector, and $\mN_{\mathrm{G}}\in \mathbb{R}^{d_Y}$ is a 
Gaussian noise vector independent of $\mX$. The input and the
noise covariance matrices are given by
$\mSigma$ and $\mSigma_{N_{\mathrm{G}}}$.
Then, the
mutual information and lautum information are given by
\begin{align}
    I(\mX;\mY) &= \mathrm{tr}\big(\mSigma_{N_{\mathrm{G}}}^{-1} \mA \mSigma \mA^\top \big) - D\big(P_\mY\|P_{N_{\mathrm{G}}} \big),  \\
    L(\mX;\mY) &= \mathrm{tr}\big(\mSigma_{N_{\mathrm{G}}}^{-1} \mA \mSigma \mA^\top \big) + D\big(P_\mY\|P_{N_{\mathrm{G}}}).
\end{align}

\end{lemma}

In our example, the output $W$ can be written as
\begin{align}
    W & = \frac{\sigma_1^2 }{\sigma_0^2}\vmu_0 +\frac{\sigma_1^2 }{\sigma^2}\sum_{i=1}^n Z_i +N, \nn \\
    & = \frac{\sigma_1^2 }{\sigma^2}\sum_{i=1}^n (Z_i-\vmu) + \frac{\sigma_1^2 }{\sigma_0^2}\vmu_0+\frac{n \sigma_1^2  }{\sigma^2}\vmu + N,
\end{align}
where $N\sim \mathcal{N}(0,\sigma^2_1 I_d)$. Then, setting $P_{N_{\mathrm{G}}} \sim \mathcal{N}(\frac{\sigma_1^2 }{\sigma_0^2}\vmu_0+\frac{n \sigma_1^2  }{\sigma^2}\vmu, \sigma^2_1 I_d)$, $\mSigma = \sigma_Z^2 I_{nd}$ and noticing that $\mA  \mA^\top = \frac{n \sigma_1^4}{ \sigma^4}I_{d}$ in Lemma~\ref{lemma:Gaussian}  completes the proof.

\subsection{ISMI bound}\label{app:ISMI}
In this subsection, we evaluate the ISMI bound from \citep{bu2020tightening} for the example discussed in Section \ref{sec:mean_example} with i.i.d. samples generated from Gaussian $P_Z \sim \mathcal{N}(\vmu, \sigma_Z^2 I_d)$.

\begin{lemma}{\citep[Theorem 2]{bu2020tightening}}\label{lemma:ISMI}
\quad Suppose $\ell(\widetilde W,\widetilde Z)$ satisfies $\Lambda_{\ell(\widetilde W,\widetilde Z)}(\lambda) \le \psi_{+}(\lambda)$ for $\lambda \in [0,b_+)$, and $\Lambda_{\ell(\widetilde W,\widetilde Z)}(\lambda) \le \psi_{-}(-\lambda)$ for $\lambda\in (b_-,0]$ under $P_{\widetilde Z,\widetilde W} = P_Z\otimes P_{W}$, where $0 <b_+\le \infty$ and $-\infty \leq b_- <0 $. Then,
\begin{align}
    \mathrm{gen}(P_{W|S},P_S) \le \frac{1}{n} \sum_{i=1}^n  \psi^{*-1}_{-}\big(I(W;Z_i)\big),\\
  -\mathrm{gen}(P_{W|S},P_S) \le \frac{1}{n} \sum_{i=1}^n  \psi^{*-1}_{+}\big(I(W;Z_i)\big).
\end{align}
\end{lemma}

First, we need to compute the mutual information between each individual sample and the output hypothesis $I(W;Z_i)$, and the cumulant generating function (CGF) of $\ell(\widetilde W,\widetilde{Z})$, where $\widetilde{W}$, $\widetilde{Z}$ are independent copies of $W$ and $Z$ with the same marginal distribution, respectively.

Since $W$ and $Z_i$ are Gaussian, $I(W;Z_i)$ can be computed exactly as:
\begin{equation}
 {\rm{Cov}}[Z_i, W] = \left(
    \begin{array}{cc}
     \sigma_Z^2I_d & \frac{\sigma_1^2}{\sigma^2} \sigma_Z^2 I_d  \\
     \frac{\sigma_1^2}{\sigma^2} \sigma_Z^2 I_d &  \big(\frac{n \sigma_1^4}{\sigma^4}\sigma_Z^2 +\sigma_1^2\big)I_d\\
    \end{array}
 \right),
\end{equation}
then, we have
\begin{align}\label{eq:mean_ISMI}
  I(W;Z_i) 
  & = \frac{d}{2} \log \frac{\frac{n \sigma_1^4}{\sigma^4}\sigma_Z^2 +\sigma_1^2}{\frac{(n-1) \sigma_1^4}{\sigma^4}\sigma_Z^2 +\sigma_1^2} \nn \\
  &= \frac{d}{2} \log\Big(1+ \frac{\sigma_1^2\sigma_Z^2}{(n-1) \sigma_1^2\sigma_Z^2 +\sigma^4}\Big) \\
  & = \frac{d}{2} \log\Big(1+ \frac{\sigma_0^2\sigma_Z^2}{(n-1) \sigma_0^2\sigma_Z^2 +n \sigma_0^2 \sigma^2+\sigma^4}\Big),\nn 
\end{align}
for $i=1,\cdots,n$, $n\ge 2$. In addition, since
\begin{equation}
    W\sim \mathcal{N}\Big(\frac{\sigma_1^2 }{\sigma_0^2}\vmu_0 +\frac{n\sigma_1^2 }{\sigma^2}\vmu, \big(\frac{n \sigma_1^4}{\sigma^4}\sigma_Z^2 +\sigma_1^2\big)I_d\Big),
\end{equation}
it can be shown that $\ell(\widetilde W,\widetilde{Z})=\|\widetilde{Z}-\widetilde{W}\|^2$ is a scaled non-central chi-square distribution with $d$ degrees of freedom, where the scaling factor $\sigma_{\ell}^2 \triangleq (\frac{n\sigma_1^4}{\sigma^4}+1)\sigma_Z^2+\sigma_1^2$ and its non-centrality parameter $\eta \triangleq \frac{\sigma^2}{n \sigma_0^2+\sigma^2}\|\vmu_0-\vmu\|_2^2$. Note that the expectation of chi-square distribution with non-centrality parameter $\eta$ and $d$ degrees of freedom is $d + \eta$ and its moment generating function is $ \exp(\frac{\eta \lambda}{1-2\lambda})(1-2\lambda)^{-d/2}$. Therefore, the CGF of $\ell(\widetilde W,\widetilde Z)$ is given by
\begin{equation*}
  \Lambda_{\ell(\widetilde W,\widetilde Z)}(\lambda) = - (d \sigma_\ell^2 +\eta) \lambda +\frac{\eta\lambda}{1-2\sigma_\ell^2 \lambda}- \frac{d}{2} \log(1-2\sigma_\ell^2\lambda),
\end{equation*}
for $\lambda \in (-\infty, \frac{1}{2\sigma_\ell^2})$.
Since $\mathrm{gen}(P_{W|S},P_Z)\ge 0$, we only need to consider the case $\lambda <0$. It can be shown that:
\begin{align}
\Lambda_{\ell(\widetilde W,\widetilde Z)}(\lambda) &= - d \sigma_\ell^2 \lambda - \frac{d}{2} \log(1-2\sigma_\ell^2\lambda) +\frac{2\sigma_\ell^2\eta\lambda^2}{1-2\sigma_\ell^2\lambda} \nn \\
& = \frac{d}{2} (-u-\log(1-u)) +\frac{2\sigma_\ell^2\eta\lambda^2}{1-2\sigma_\ell^2\lambda},
\end{align}
where $u\triangleq2\sigma_\ell^2\lambda$. Further note  that
\begin{align}
-u-\log(1-u) \le \frac{u^2}{2},\ u&<0,\\
\frac{2\sigma_\ell^2\eta\lambda^2}{1-2\sigma_\ell^2\lambda} \le 2\sigma_\ell^2\eta\lambda^2,\ \lambda&<0.
\end{align}
We have the following upper bound on the CGF of $\ell(\widetilde{W},\widetilde{Z})$:
\begin{equation}
\Lambda_{\ell(\widetilde W,\widetilde Z)}(\lambda) \le (d\sigma_\ell^4+2\sigma_\ell^2\eta)\lambda^2,\quad \lambda<0,
\end{equation}
which means that $\ell(\widetilde W,\widetilde Z)$ is $\sqrt{d\sigma_\ell^4+2\sigma_\ell^2\eta}$-sub-Gaussian for $\lambda<0$.
Combining the results in \eqref{eq:mean_ISMI}, Lemma~\ref{lemma:ISMI} gives the following bound
\begin{align*}
  &\overline{\text{gen}}(P_{W|S},P_S)\le\\\nn
  &\quad\sqrt{\frac{d^2 \sigma_{\ell}^4 +2d \sigma_{\ell}^2\eta}{2} \log(1+ \frac{\sigma_0^2\sigma_Z^2}{(n-1) \sigma_0^2\sigma_Z^2 +n \sigma_0^2 \sigma^2+\sigma^4})}.
\end{align*}
If $\sigma^2=\frac{n}{2\alpha}$ is fixed, i.e., $\alpha = \mathcal{O}(n)$, then as $n \to \infty$, $\sigma_{\ell}^2 = \mathcal{O}(1)$, and the above bound is $\mathcal{O}\left(\frac{1}{\sqrt{n}}\right)$.

\section{Proof of Theorem~\ref{Theorem: Sub Gaussian extension}}\label{app: Average Upper Bound Details}

Using the \citep[Theorem~1]{xu2017information} and considering the non-negativity of generalization error, we have:
\begin{align}
    \overline{\text{gen}}(P_{{W}|S},\mu) \leq \sqrt{\frac{2\sigma^2 I(W;S)}{n}}
\end{align}
By the Theorem~\ref{Theorem: Gibbs Result}, the following holds:
\begin{equation}
     \frac{I_{\mathrm{SKL}}(W;S)}{\alpha }  \leq \sqrt{\frac{2\sigma^2 I(W;S)}{n}}
\end{equation}
As the $C_E \times I(W;S)\leq L(W;S)$ we have:
\begin{equation}
     \frac{(1+C_E) I(W;S)}{\alpha } \leq \sqrt{\frac{2\sigma^2 I(W;S)}{n}}
\end{equation}
which is true for
\begin{equation}
  \sqrt{\frac{ 2\sigma^2 I(W;S)}{n}}\leq \frac{2\sigma^2\alpha}{n(1+C_E)}.
\end{equation}
It completes the proof.

\end{document}

%% file: math_commands.tex
\usepackage{amsmath,amsfonts,bm}

\def\1{\bm{1}}

\def\vmu{{\bm{\mu}}}

\def\vw{{\bm{w}}}

\def\vz{{\bm{z}}}

\def\mA{{\bm{A}}}

\def\mN{{\bm{N}}}

\def\mX{{\bm{X}}}
\def\mY{{\bm{Y}}}

\def\mSigma{{\bm{\Sigma}}}

\DeclareMathAlphabet{\mathsfit}{\encodingdefault}{\sfdefault}{m}{sl}
\SetMathAlphabet{\mathsfit}{bold}{\encodingdefault}{\sfdefault}{bx}{n}

